\documentclass[letterpaper, 10 pt, conference, left=48pt, right=48pt, bottom=43pt, top=57pt]{ieeeconf}

\IEEEoverridecommandlockouts                              

\usepackage{geometry}
\usepackage{graphicx} 
\usepackage{amsmath} 
\usepackage{amssymb}  

\usepackage{amsthm}
\usepackage{soul}
\usepackage[dvipsnames]{xcolor}
\usepackage[export]{adjustbox} 

\newcommand{\green}[1]{{\normalsize{{\color{ForestGreen}#1}}}}

\theoremstyle{definition}
\newtheorem{definition}{Definition}
\newtheorem{theorem}{Theorem}

\newtheorem{problem}{Problem}
\newtheorem{remark}{Remark}

\usepackage{cancel}
\usepackage{wrapfig}
\usepackage{cuted}
\usepackage{subcaption}

\usepackage{caption}
\usepackage{dsfont}
\DeclareCaptionType{equ}[][]
\graphicspath{{figures/}} 
\usepackage{hyperref}
\usepackage{bbm}
\usepackage{comment}
\usepackage{multirow}
\title{ \LARGE \bf
Feasible Space Monitoring for Multiple Control Barrier Functions with application to Large Scale Indoor Navigation}

\author{Hardik Parwana$^{1}$, Mitchell Black$^{2}$, Bardh Hoxha$^{2}$, Hideki Okamoto$^{2}$, Georgios Fainekos$^{2}$, \\Danil Prokhorov$^{2}$, Dimitra Panagou$^{1,3}$ 
    \thanks{This work was mostly performed while Hardik Parwana was at Toyota. 
    }
    \thanks{$^{1}$ Department of Robotics, $^{3}$ Department of Aerospace Engineering, University of Michigan, Ann Arbor, USA {\tt\small \{hardiksp, dpanagou\}@umich.edu}
    }
 \thanks{$^{2}$ Toyota Motor North America, Research \& Development, Ann Arbor, MI 48105, USA
 {\tt\small {<first\_name.last\_name> }@toyota.com.}
 }
}
\newcommand{\reals}{\mathbb{R}}
\newcommand{\s}{\mathcal{S}}
\newcommand{\X}{\mathcal{X}}
\newcommand{\U}{\mathcal{U}}

\newcommand{\K}{\mathcal{K}}
\newcommand{\classK}{class-$\K$ }

\newcommand{\Int}{\text{Int} }

\newcommand{\eqn}[1]{\begin{align}#1\end{align}}

\begin{document}

\maketitle
\thispagestyle{empty}
\pagestyle{empty}

\begin{abstract}
    Quadratic programs (QP) subject to multiple time-dependent control barrier function (CBF) based constraints have been used to design safety-critical controllers. However, ensuring the existence of a solution at all times to the QP subject to multiple CBF constraints (hereby called compatibility) is non-trivial. We quantify the feasible control input space defined by multiple CBFs at a state in terms of its volume. We then introduce a novel feasible space (FS) CBF that prevents this volume from going to zero. FS-CBF 
    is shown to be a sufficient condition for the compatibility of multiple CBFs. For high-dimensional systems though, finding a valid FS-CBF may be difficult due to the limitations of existing computational hardware or theoretical approaches. 
    In such cases, we show empirically that imposing the feasible space volume as a candidate FS-CBF not only enhances feasibility but also exhibits reduced sensitivity to changes in the user-chosen parameters such as gains of the nominal controller. Finally, paired with a global planner, we evaluate our controller for navigation among other dynamically moving agents in the AWS Hospital gazebo environment. The proposed controller is demonstrated to outperform the standard CBF-QP controller in maintaining feasibility.
    
\end{abstract}

\section{INTRODUCTION}
\label{section::introduction}
Designing a safety-critical robot motion stack for ensuring the satisfaction of multiple, possibly time-dependent, state and input constraints is an active area of research. This work deals with the formulation of low-level controllers for control-affine dynamical systems based on control barrier functions (CBFs). Given a constraint set, a CBF for given robot dynamics can be designed such that imposing a restriction on its rate-of-change in the control design guarantees constraint satisfaction for all times. This condition is typically imposed on the control input in an optimization problem such as a quadratic program (QP). In the presence of multiple constraint sets, it is common to simultaneously impose CBFs designed independently for each constraint in the controller. 
However, in such implementations, no guarantees for existence of a control satisfying all the CBF constraints exist, and violation of safety may result.

In this work, we introduce a CBF that enforces the existence of a common solution to multiple CBF constraints at all times, a property we call persistent compatibility.
Towards this, we introduce the following new metric. 
At a given state and time, we quantify the feasible control input space defined by multiple CBF constraints and input bounds in terms of its volume. For example, for control affine dynamics and polytopic input constraints, the feasible control space is also given by a convex polytope. 
Next, we introduce a novel feasible-space control barrier function (FS-CBF) that prevents this volume from going to zero. Our controller is thus a QP that augments the user-provided CBF constraints with an additional FS-CBF constraint.
The existence of FS-CBF is shown to be a sufficient condition for enforcing the persistent compatibility of multiple CBFs. 

Two commonly used approaches employed in the literature for handling multiple CBF constraints are:
combining all barrier functions into a single one \cite{panagou2013multi, glotfelter2017nonsmooth, stipanovic2012monotone} or imposing the barriers simultaneously \cite{usevitch2020strong,aali2022multiple,wang2022ensuring}. 
In \cite{panagou2013multi, stipanovic2012monotone}, barrier functions are combined into a single barrier constraint through a smoothed minimum (or maximum) operator. Employing non-smooth analysis, \cite{glotfelter2017nonsmooth} constructs a barrier function with the non-smooth minimum (or maximum) operation. All of the aforementioned approaches assume compatibility of CBFs when imposed together \cite{usevitch2020strong, aali2022multiple} or assume that the newly constructed CBF is a valid CBF for dynamics of the system and input bounds.


A few studies explicitly try to circumvent the persistent compatibility issue. A sampling-and-grid refinement method is proposed in \cite{tan2022compatibility} to search the domain of state space over which multiple control barrier functions are compatible. However, it does not propose a controller to confine the robot to this domain. 
\begin{figure*}[t]
    \centering
    \includegraphics[width=0.8\linewidth]{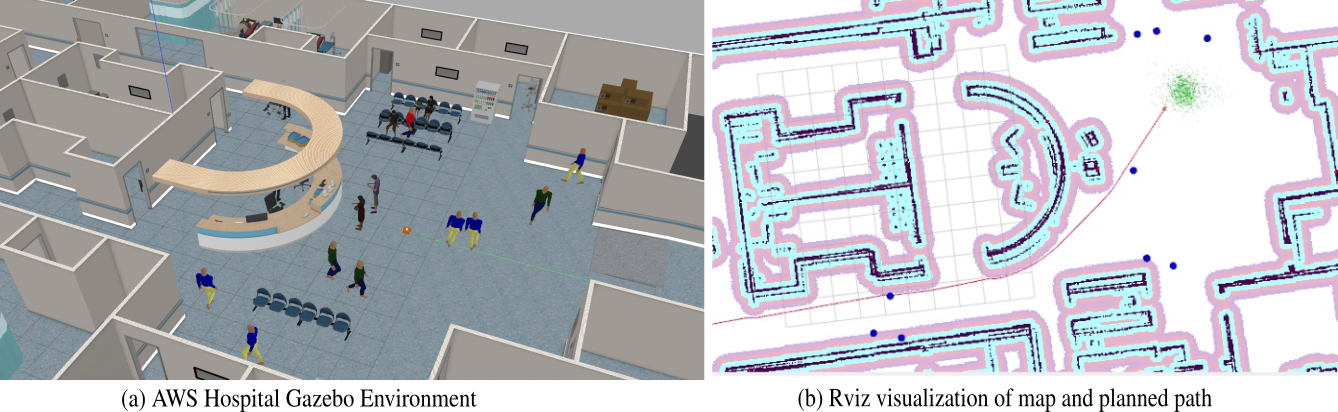}
    \captionof{figure}{\small{Indoor navigation scenario (a) Robot (green circle) navigating in AWS Hospital Gazebo environment using proposed CBF controller. (b) Rviz visualization of the global map, global planner's path (green), humans (blue) and the nearest static obstacles (red, found in the way described in Section \ref{section::case_study_3}) used by the robot for collision avoidance.}}
    \label{fig::experiment_scenario}
\end{figure*}
A sum-of-squares (SoS)-based CBF synthesis method when constraint sets are described by polynomial functions is given in  \cite{clark2021verification}. In
\cite{breeden2023compositions}, sufficient conditions are provided for designing multiple CBFs concurrently so that a QP-based controller using all the CBFs will always be feasible. The work by \cite{black2023adaptation} consolidates multiple CBFs into one CBF and proposes a predictor-corrector adaptation law for weights of constituent constraint functions. 
An offline-sampling procedure is proposed in \cite{lee2024data} to construct a CBF from constraint functions when the former is related to the latter through a constant scaling and offset of state vector. Additionally, a CBF learning framework is proposed in \cite{qin2021learning} for decentralized but cooperative multi-agent systems. Some works allow \classK functions to adapt by having time-varying parameters. The work by \cite{xiao2020feasibility} creates a dataset of persistently compatible and eventually incompatible parameters by randomly sampling parameters and simulating the system forward. Adaptation of \classK parameters is also performed by posing CBF-QP as a differentiable layer in a predictor-corrector framework \cite{parwana2022recursive} and in learning paradigm \cite{xiao2023barriernet} for learning of safe policies. 

We introduce a procedure that exploits existing approaches such as reachability analysis \cite{tonkens2022refining} to synthesize a valid FS-CBF.
Owing to the limitations of these approaches though, it might not always be computationally feasible to find FS-CBF. In such cases, we show empirically that
that imposing feasible space volume as a candidate FS-CBF enhances the persistent compatibility of standard CBF formulations. It is also observed to reduce the sensitivity of CBF-QP controllers to user-chosen parameters such as gains of the nominal controller. 
Furthermore, finding the volume of feasible space and its gradient with respect to constituent hyperplanes are non-trivial problems as they do not accept analytical expressions in general. As a contribution, we also introduce some of the existing approximations suitable for control-oriented applications.
To the best of our knowledge, we are the first to present this paradigm of quantifying the feasible control input space of CBFs as well as consider an optimization problem's feasible solution space volume and its gradient in control design. 
Previous approaches simply check if the CBF-QP is feasible, which amounts to checking whether the volume is non-zero. 
Note that instead of synthesizing FS-CBF, it may be possible to consolidate multiple CBFs into a single valid CBF offline. 
However, such an ideal scenario is difficult, if not impossible, to achieve due to prohibitive computational requirements of existing approaches. 
It may also be difficult to account for all possible possibilities in the offline phase. 
For example, such scenarios may have varying number of humans and obstacles. Our motivation lies in ensuring safety in highly constrained indoor environments, such as in Fig. \ref{fig::experiment_scenario} with static as well as dynamically moving agents like humans. 
Our idea to monitor feasible space volume thus has a high practical appeal for robotic applications.


Finally, as a case study, paired with ROS2 navigation stack planners, we also provide the 
application of CBF-QP controller on a large scale indoor navigation scenario by implementing our controller on AWS Hospital environment in Gazebo \cite{awsHospital}. 
Several existing approaches such as \cite{silva2023online,patompak2016mobile,cai2023sampling} for indoor navigation with humans rely only on fast replanning through a low-level planner based on simplified dynamics for collision avoidance, and do not employ a safety-critical controller.
Our controller may thus complement existing planning frameworks. 
Among CBF-based approaches, \cite{vulcano2022safe} uses CBF-MPC, an MPC controller with constraints defined by CBF derivative conditions, for navigation in crowd, and \cite{thirugnanam2022safety} designs a novel CBF for polytopic obstacles and uses it in the CBF-MPC framework for navigation in corridor environments. From an application perspective, \cite{vulcano2022safe} does not consider indoor environments and \cite{thirugnanam2022safety} does not consider dynamic moving agents. While MPC does help with maintaining feasibility of CBF constraints, it is unlike our method of monitoring the volume of feasible space.

\subsubsection*{Summary of Contributions}
\begin{itemize}
\item Introduced the notion of persistent compatibility and formulated Feasible-Space Control Barrier Function (FS-CBF) to quantify and restrict the rate-of-change of the volume of the feasible control input space.
\item Demonstrated through simulations that FS-CBF enhances persistent feasibility and reduces hyperparameter sensitivity in CBF-QP controllers. We also demonstrate our controller in a large-scale indoor navigation scenario, integrated with ROS2 navigation stack planners.
\end{itemize}

In the subsequent, we formulate our problem in Section \ref{section::problem_formulation}, introduce feasible space, methods to compute its volume, and design of FS-CBF in Section \ref{section::fs_cbf}. Finally, we present simulation results in Section \ref{section::simulation_results}.
\section{Notation}
The set of real numbers is denoted as $\reals$. The interior and boundary of a set $\mathcal{C}$ is denoted by $\textrm{Int}(\mathcal{C})$ and $\partial \mathcal{C}$. The empty set is denoted by $\emptyset$. For $a\in \reals^+$, a continuous function $\alpha:[0,a)\rightarrow[0,\infty)$ is a \classK function if it is strictly increasing and $\alpha(0)=0$. For $x\in \reals, y\in \reals^n$, $|x|$ denotes the absolute value of $x$ and $||y||$ denotes the $L_2$ norm of $y$. The time derivative of $x$ is denoted by $\dot x$ and $\frac{\partial F}{\partial x}$ denotes the gradient of a function $F: \reals^n \rightarrow \reals$ with respect to $x\in \reals^n$.

\section{System Description}
\label{section::system_description}
Consider a system with state $x\in \X \subset \reals^{n}$, control input $u \in \U \subset \reals^{m}$ with the following dynamics
\eqn{
\dot{x} = f(x) + g(x)u,
\label{eq::dynamics_general}
}
where $f:\X\rightarrow \reals^{n}$ and $g:\X\rightarrow \reals^{n \times m}$ are locally Lipschitz continuous functions and $x(t)$ is the solution of \eqref{eq::dynamics_general} at time $t$ from initial condition $x(0)=x_0$. The set of allowable states $\s(t)$, hereby called safe set, at time $t$ is specified as an intersection of $N$ sets $\mathcal S_i(t),i\in\{1,2,..,N\}$. Each $\mathcal S_i(t)$ is defined as the 0-superlevel set of a continuously differentiable function $h_i:\reals^+ \times \mathcal{X} \rightarrow \reals$, i.e.,
\begin{subequations}
    \begin{align}
        \s_i(t) & \triangleq \{ x \in \X : h_i(t,x) \geq 0 \}, \label{eq::safeset1} \\
        \partial \s_i(t) & \triangleq \{ x\in \X: h_i(t,x)=0 \}, \label{eq::safeset2}\\
        \Int (\s_i)(t) & \triangleq \{ x \in \X: h_i(t,x)>0  \}. \label{eq::safeset3}
    \end{align}
    \label{eq::safeset}
\end{subequations}
When the relative degree $r_i$ of $h_i$ is greater than one, we can derive the following new barrier functions
\eqn{
\psi^{k}_i(t,x) = \dot \psi^{k-1}_i(t,x) + \alpha_i^k \psi_i^k(t,x)
}
for $k=\{1,..,r_i-1\}$, $\psi_i^0=h_i$ and where $\alpha_i^k$ are extended \classK functions.

\begin{definition}(Higher Order CBF)\cite{xiao2019control,xiao2021high}
\label{definition::hocbf_definition}
    The function $h_i(t,x):\reals^+\times \reals^{n}\rightarrow \reals$ is a Higher-Order CBF (HOCBF) of $r_i$-th order on the set $\s_i(t)$ if there exist $r_i$ extended $\mathcal{K}$ functions $\alpha_i^k:\reals \rightarrow \reals, \;  k\in\{1,2,..,r_i\}$, and an open set $D_i(t)\subset \reals^+ \times \reals^n$ with $\s_i(t) \subset D_i(t) \subset \mathcal{X}$ such that
     \eqn{
      \dot \psi_i^{r_i-1}(t,x,u) \geq -\alpha_i^{r_i} (\psi_i^{r_i-1}(t,x)), ~ \forall x \in D_i(t), \forall t\geq 0.
      \label{eq::cbf_derivative}
    }
\end{definition}

Note that for $r_i=1$, \eqref{eq::cbf_derivative} reduces to $\dot h_i (x,u)\geq -\alpha_i h_i(x)$. And for $r_i>1$, $\dot \psi_i^{k}$ does not depend on $u$ for $k<r_i-1$. It is assumed in this work that each $h_i$ is a HOCBF for the set $\mathcal{S}_i$ and that there exists a Lipschitz continuous controller $u:\mathbb R^+\times\mathcal S_i\rightarrow \mathcal U$ that satisfies \eqref{eq::cbf_derivative}, guaranteeing thus the forward invariance of the set $\mathcal{S}_i$ under the closed-loop dynamics \eqref{eq::dynamics_general} \cite{ames2016control,lindemann2018control}.
Finally, we review the following controller that is often used to enforce multiple CBF constraints
\begin{align}
\label{eq::qp_controller}
u(t,x) = \arg \min_{\green{u\in \mathcal
U}} \quad &\! ||u - u_{ref}(t,x)||^2  \\
          \textrm{s.t.} \quad & \dot \psi_i^{r_i-1}(t,x,u) \geq \alpha_i^{r_i}(\psi_i^{r_i-1}(t,x)) \nonumber \\
          & \quad \quad \quad \quad \quad \quad \quad \quad i\in \{1,..,N\} \nonumber
\end{align}
where $u_{ref}:\mathbb R^+\times\mathcal X\rightarrow \mathcal U$ is a reference control policy.

\section{Problem Formulation}
\label{section::problem_formulation}
Given a barrier function $h_i, i\in \{1,..,N\}$, the set of control inputs that satisfy the CBF derivative constraint \eqref{eq::qp_controller} for a given state $x$ at time $t$ are given by
\eqn{
  \mathcal{U}_i(t,x) = \{ u \in \reals^m ~|~ \dot \psi_i^{r_i-1}(t,x,u) + \alpha_i^{r_i}(\psi_i^{r_i-1}(t,x)) \geq 0\}.
  \label{eq::feasible_cbf1}
}
\begin{definition}
    \textbf{Feasible CBF Space}: At a given time $t$ and state $x\in \mathcal{X}$, the feasible CBF space $\mathcal{U}_c \subset \reals^m$ is the intersection of sets $\mathcal{U}_i, i\in\{1,..,N\}$ with control input domain
    \eqn{
    \mathcal{U}_c(t,x) = \bigcap\limits_{i=1}^{N} ~ \mathcal{U}_i(t,x) \bigcap \mathcal{U}.
    \label{eq::cbf_control_intersection}
}
\end{definition}
For control-affine dynamics, the constraints \eqref{eq::feasible_cbf1} are affine in $u$ and represent halfspaces. 
Note that for control-affine dynamics, \eqref{eq::cbf_derivative} and \eqref{eq::feasible_cbf1} are affine in the control input $u$, and therefore \eqref{eq::qp_controller} is a QP controller when $\mathcal{U}=\reals^m$. 
If the input domain is also polytopic, that is, $\mathcal{U}=\{u \in \reals^m ~|~ A_u u \leq b_u\}, A_u\in \reals^{2m\times m}, b\in \reals^{2m}$, then the feasible space $\mathcal U_c$ is a convex polytope, and the controller \eqref{eq::qp_controller} is still a QP, since only linear constraints have been added.
Figure~\ref{fig::si_fs} visualizes the polytope for a single integrator-modeled robot as it navigates around obstacles. For brevity, we will be denoting the feasible space in the following form:
\eqn{
\mathcal{U}_c(t,x) = \{ u\in \reals^m | A(t,x) u \leq b(t,x)\}
\label{eq::cvx_polytope_eqn}
}
where $A: \reals^+ \times \reals^n \rightarrow \reals^{(N+2m)\times m}, b:  \reals^+ \times \reals^n \rightarrow \reals^{(N+2m)\times 1}$. We also denote each row of $A$ by $a_i$ and each element of $b$ by $b_i, i\in \{1,2,..,N\}$. Next, we formally introduce the notion of compatibility and state the objective of this paper. 
\begin{figure}
    \centering
    \includegraphics[width=0.45\textwidth]{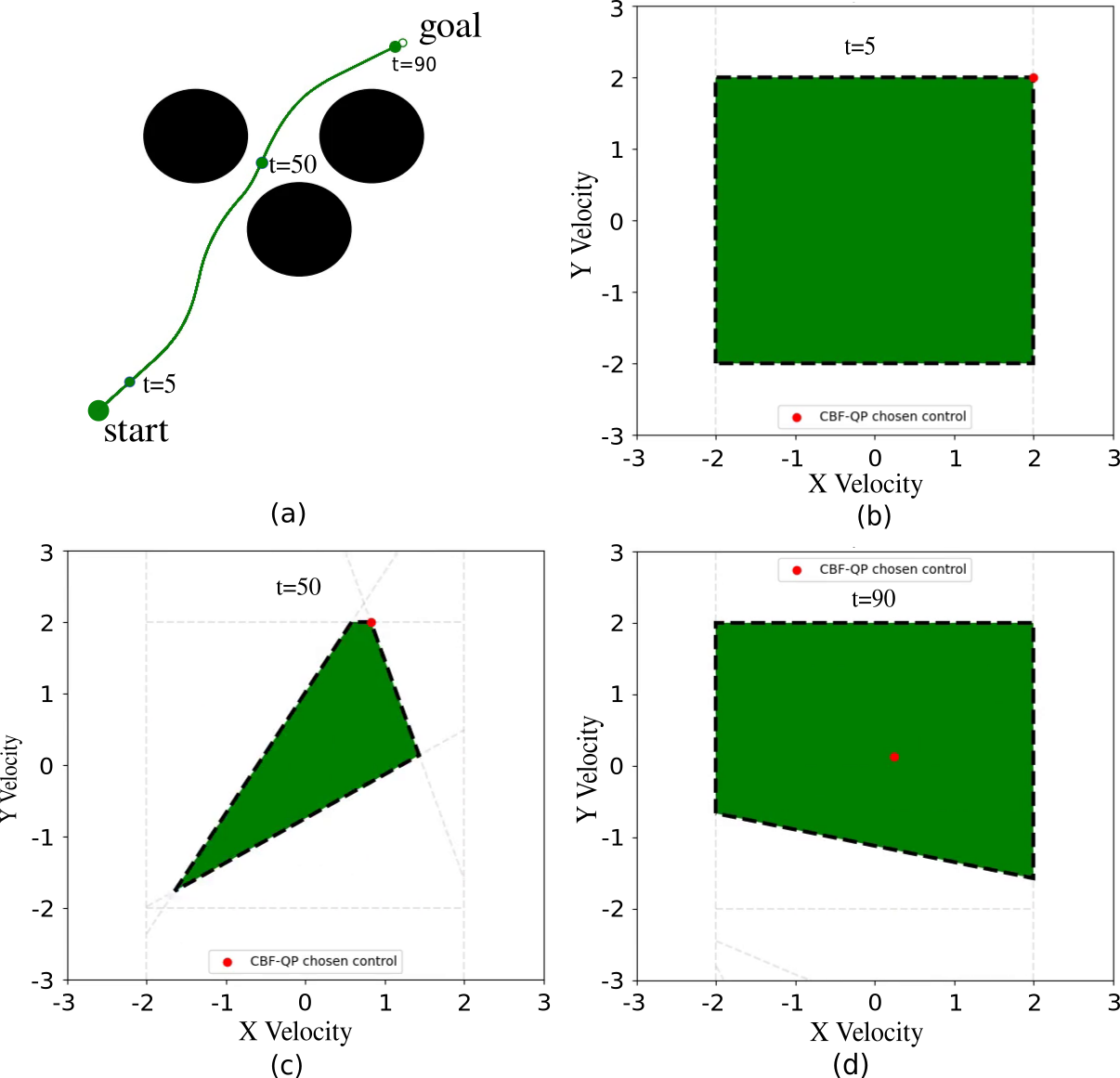}
    \caption{\small{Feasible space (green) of a single integrator agent with X, Y velocity inputs (in m/s) as it navigates to its goal location with N=3 in \eqref{eq::qp_controller} while avoiding obstacles (black). (a) The path in configuration space, (b), (c), (d) feasible control space visualization at different times. The dotted lines represent hyperplane equations that constrain the control input. The $2\times 2$ square is formed by control input bound. Other hyperplanes are formed as a result of CBF equations.}}
    \label{fig::si_fs}
\end{figure}

\begin{definition}
    \textbf{(Compatible CBFs)} At time $t$ and state $x$, the barriers $h_i, i\in \{1,..,N\}$ are called compatible if $\mathcal{U}_c(t,x)\neq \emptyset$. Furthermore, given a Lipschitz continuous controller $\pi : \mathcal{X} \rightarrow \mathcal{U}$, the barriers $h_i$ are called persistently compatible for dynamics
    \eqref{eq::dynamics_general} with $u = \pi(x)$ in \eqref{eq::dynamics_general} if $\exists E\subset \mathcal{X} $ such that $x_0\in E \implies \mathcal{U}_c(t,x)\neq \emptyset$ for all $t\geq 0$.
    \label{definition::persistently_feasible}
\end{definition}

\begin{problem}
    Consider the dynamical system \eqref{eq::dynamics_general} subject to $N$ barrier function constraints $h_i(t,x)\geq 0$ where each $h_i$, defined in \eqref{eq::safeset}, is a CBF with \classK function $\alpha_i$ for the set $\mathcal{S}_i$ respectively. If for the initial state $x_0\in \reals^n$, $h_i(t,x_0)>0, i\in\{1,..,N\},~ \mathcal{U}_c(0,x_0)\neq \emptyset$, design a controller such that $\forall t>0,~ h_i(t,x(t))>0$ and $\mathcal{U}_c(t,x)>0 $.
    \label{objective::1}
\end{problem}

We propose a solution to Problem \ref{objective::1} in Section \ref{section::FS-CBF-theory}. However, owing to limitations of related literature, our solution may be difficult to realize in practice. We therefore propose several approximations in Sections \ref{section::candidate-FS-CBF} and \ref{section::volume_estimation} and provide empirical validation of their efficacy in Sections \ref{section::case_study_2} and \ref{section::case_study_3}.

\section{Feasible Volume Control Barrier Function}
\label{section::fs_cbf}
The infeasibility of CBF-QP results when $\mathcal{U}(t,x)$ becomes empty. In this section, we propose a novel control barrier function built upon the feasible space as a means to solving Problem \ref{objective::1}. 

\subsection{Polytope Volume Barrier Function}
\label{section::FS-CBF-theory}
A measure of the feasible space is its volume. Let $\mathcal V(t,x) = \textrm{vol}(\mathcal U_c)(t,x)$ where $\textrm{vol}$ is the volume function that, for a polytope $P=\{u | Au\leq b\}$, is defined as follows
\eqn{
\textup{vol}(P) = \int\limits_{u\in P}d\mathcal{V}, ~~ d\mathcal{V}=\prod_{j=1}^m du_j
\label{eq::polytope_volume}
} 
where $du_j$ is an infinitesimal change in $u_j$. 
Note that $\mathcal{V}(t,x)=0$ when $\mathcal{U}_c(t,x)=\emptyset$. The compatible state space is defined as the set of states $D(t)=\{x\in \mathcal{X} ~|~ \mathcal V(t,x)\geq 0\}$ and is visualized in Fig.~\ref{fig::mainfig}. 
We then aim to find a CBF $\mathcal{V}_c: \reals^+ \times \reals^n\rightarrow \reals$, with \classK function $\alpha_{\mathcal V_c}$,  that renders $D(t)$ forward invariant by imposing the following condition
\eqn{
    \dot{\mathcal{V}_c}(t,x,u) \geq -\alpha_{\mathcal{V}_c}(\mathcal{V}_c(t,x))
    \label{eq::fs_cbf_condition}
}


We now present the following theorem that states the sufficiency of FS-CBF condition \eqref{eq::fs_cbf_condition} for compatibility of $h_i$.

\begin{theorem}
    Let $h_i, i\in \{1,..,N\}$ be $N$ barrier functions corresponding to $N$ safe sets defined as in \eqref{eq::safeset}. Let $\mathcal{V}(t,x)$ be the volume of the feasible CBF space $\mathcal{U}_c(t,x)$ in the control domain defined in \eqref{eq::cbf_control_intersection}. For $0<\epsilon<<1$, let $D(t)=\{x \in \mathcal{X}~ | ~\mathcal{V}(t,x)\geq \epsilon\}$ be the compatible state space of $h_i, i\in \{1,..,N\}$. Then $h_i$ are persistently compatible if there exists $\mathcal{C}\subset \mathcal{D}$ and a function $\mathcal{V}_c: \reals^+ \times \reals^n\rightarrow \reals$ such that $\mathcal{V}_c$ is a control barrier function on $\reals^+ \times \mathcal{C}$ and $x(0)\in \mathcal{C}$. 
    \label{theorem::fs-cbf}
\end{theorem}

\begin{proof}
    Suppose that $\mathcal{V}_c$ is a CBF with \classK function $\alpha_{\mathcal V}$ on $\reals^+ \times \mathcal{C}$, with $\mathcal{C}\subset \mathcal{D}$. 
    If $x(0)\in \mathcal{C}$, then $x(t)\in \mathcal{C}, \forall t>0$ under the action of the controller that satisfies \eqref{eq::fs_cbf_condition}. Therefore, $\mathcal{V}_c(t,x(t))\geq 0, \forall t> 0$ and hence $\mathcal{V}(t,x(t))\geq \epsilon$ for all $t>0$, implying persistent compatibility.
    \vspace{-1mm}
\end{proof}

Note that sets of measure 0 have zero volume. Therefore, in Theorem \ref{section::fs_cbf}, we enforce $\mathcal{V}\geq \epsilon$ to distinguish between the case when  $\mathcal{U}_c=\emptyset$ from the case when $\mathcal{U}_c$ is a set of measure 0 (for example, a line in 3-dimensional space) as the latter does not exclude compatibility. Also, note that \eqref{eq::polytope_volume} implies that when $\mathcal{U}_c=\emptyset$, $\mathcal U{(t,x)}=0$ and therefore $\mathcal{V}$ is well-defined in the entire state space.

Computing $\mathcal{V}(t,x)$ requires computing the volume of the feasible space $\mathcal{U}_c$. This is not a trivial problem and even when feasible space is shaped as a convex polytope as in \eqref{eq::cvx_polytope_eqn}, it does not admit analytical expressions for $\mathcal{V}$. In the next section, we introduce some existing approximate approaches that are suitable for our application. 

\begin{figure*}[!htb]
    \begin{minipage}{0.73\textwidth}
     \centering
    \includegraphics[width=1.0\linewidth]{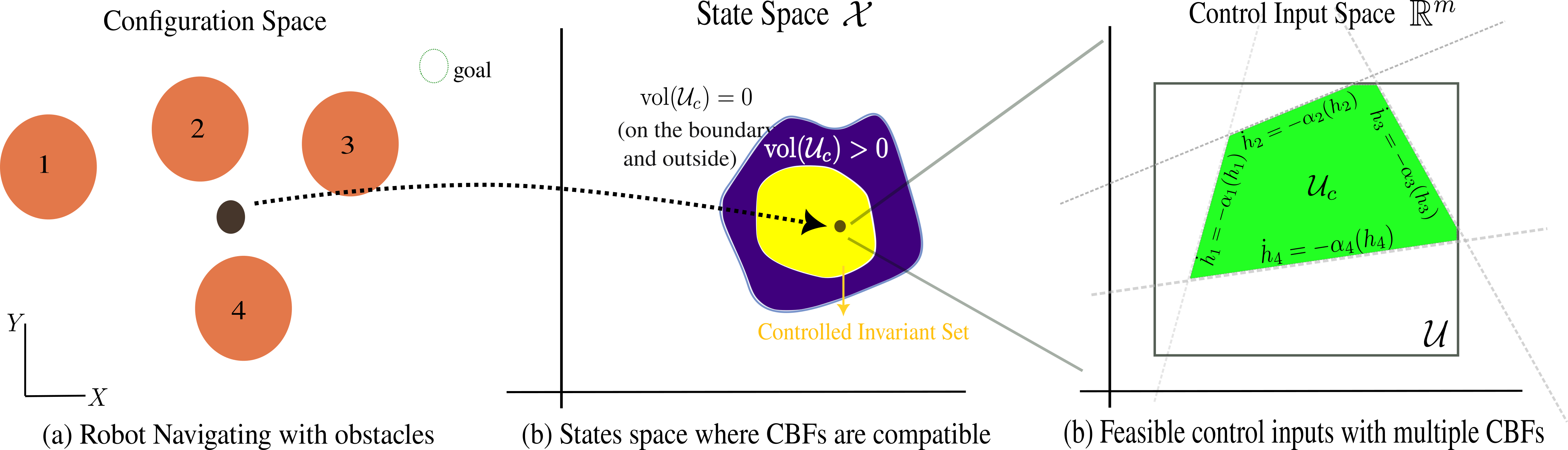}
    \captionof{figure}{\small{(a) Workspace of black robot with illustrative orange obstacles, (b) The state space. In the violet region, CBFs are compatible, i.e., the volume of the polytope $\mathcal{U}_c$ is $>0$. The yellow region is the largest control invariant set contained in the violet region. (c) The feasible control space $\mathcal{U}_c$ defined by multiple CBFs at the current state and time. It's volume is given by the area of the green region in this 2D example.
    }}
    \label{fig::mainfig}
\end{minipage}
\begin{minipage}{0.23\textwidth}
    \includegraphics[width=0.7\linewidth]{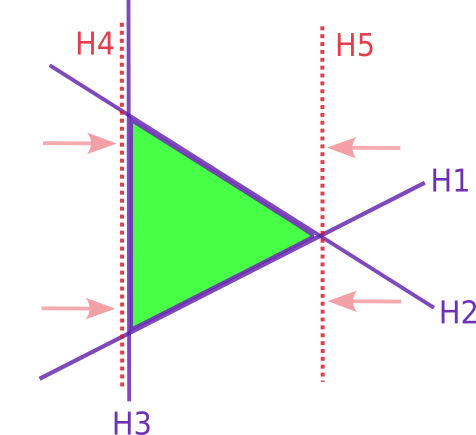}
    \captionof{figure}{\small{Feasible space (green) defined by hyperplanes H1-H5. H3, H4 are overlapping. Suppose H4, H5 are free to move horizontally. It can be shown that the gradient of the area of the green region exists w.r.t position of H5 but not w.r.t H4. \quad\quad\quad\quad \quad\quad\quad\quad\quad \quad\quad\quad\quad\quad 
    }}
    \label{fig::fs_example_classK}
\end{minipage}
\vspace{-2mm}
\end{figure*}



Contrary to intuition, only imposing \eqref{eq::fs_cbf_condition} in a controller does not ensure that the chosen control input lies inside the feasible space. Since instantaneously violating a CBF constraint does not lead to an immediate violation of safety, choosing a control input outside $\mathcal{U}_c(t,x)$ may help in increasing $\mathcal{V}_c$ and therefore satisfies the condition \eqref{eq::fs_cbf_condition}. $\mathcal{V}_c$ is strictly used to enforce compatibility of $h_i$. To ensure that chosen control input maintains compatibility of $h_i$ as well as lies inside the feasible space $\mathcal{U}_c$, we design the following controller
\begin{subequations}
    \eqn{
    \pi (t,x)\!=\! \arg\min_{u\in\mathcal{U}} \quad & ||u - u_{ref}(t,x)||^2 + M \delta^2\\
      \textrm{s.t.} \quad 
      & \dot{\mathcal{V}_c}(t,x,u) \!\geq -\alpha_{\mathcal{V}_c}(\mathcal{V}_c(t,x)) + \delta \label{eq::fs_cbf_condition_final_controller} \\
      & \psi_i^{r_i-1}(t,x,u) \geq \alpha_i^{r_i}(\psi_i^{r_i-1}(t,x)) \nonumber  \\
          & \quad \quad \quad \quad \quad \quad \quad \quad i\in \{1,..,N\} \label{eq::hocbf_condition_final_controller}
}
\label{eq::fs_final_cbf_controller}
\end{subequations}
where $\delta$ is a slack variable used to relax \eqref{eq::fs_cbf_condition_final_controller} and $M>>1$. The relaxation is necessary as maintaining compatibility of CBF constraints is not equivalent to ensuring safety. 
To see this, note that CBF constraints are mathematically well-defined even when the state is unsafe, that is, $x\notin S_i ~\forall i$, and act like a Lyapunov function to attract the state $x$ to the safe set. 
Therefore, since choosing a control input that satisfies \eqref{eq::hocbf_condition_final_controller} would, loosely speaking, ensure safety (see Remark \ref{remark::continuity}), and \eqref{eq::fs_cbf_condition_final_controller} does not have a one-to-one relationship with safety, \eqref{eq::fs_cbf_condition_final_controller} and \eqref{eq::hocbf_condition_final_controller} need not have a common solution. 
To prioritize safety, we enforce \eqref{eq::fs_cbf_condition} along with \eqref{eq::cbf_derivative} in \eqref{eq::fs_final_cbf_controller} but with a slack variable. In theory, therefore, because of $\delta$ \eqref{eq::fs_final_cbf_controller} does not possess guarantees of compatibility and could lead to infeasibility. 
Finding conditions under which compatibility and safety can be maintained simultaneously is subject to future work. 
This work, nonetheless provides a novel perspective on CBF-QP controllers and in Section \ref{section::simulation_results}, we focus on empirical showing improvements over standard CBF-QP \eqref{eq::qp_controller}.





\begin{remark}
\label{remark::continuity}
    Note that \eqref{eq::qp_controller} and  \eqref{eq::fs_final_cbf_controller} may not result in Lipschitz-continuous control inputs. To the best of our knowledge, obtaining a holistic solution that guarantees smoothness and uniqueness of solutions and ensures safety along with compatibility of constraints for \eqref{eq::dynamics_general} is an open problem. Further investigation of safety in presence of discontinuous control input, similar to \cite{usevitch2020strong}, is subject to future work.
\end{remark}

\begin{remark}
     We demonstrate in Section \ref{section::case_study_2} and Fig. \ref{fig::cs1_scenario} that imposing constraint \eqref{eq::fs_final_cbf_controller} results in $u$ being chosen in the interior of $\mathcal{U}_c$, even when the nominal control input is outside the feasible space. This is done to prevent the volume from shrinking too fast. The controller \eqref{eq::qp_controller} on the hand always chooses a point on the boundary of feasible space. 
     A mapping from $u_{ref}$ to the interior of feasible space was also designed in \cite{yang2022differentiable} through a gauge map. The gauge map was designed in a different way than ours and focused on obtaining a differentiable control design with single barrier constraint only. 
     \label{remark::gauge_map}
     \vspace{-2mm}
\end{remark}

\subsection{Polytope Volume as Candidate CBF}
\label{section::candidate-FS-CBF}

In Section \ref{section::case_study1}, we show an application of an existing CBF synthesis method to construct $\mathcal V_c$ for a simple system. For highly dynamic and constrained environments though, as previously mentioned, it is usually computationally prohibitive to consider all possible scenarios in the design of CBFs. Therefore, as is typically done in practice, we assume the constraint function is the CBF, that is, $\mathcal{V}_c=\mathcal{V}$ and use the condition \eqref{eq::fs_cbf_condition} to promote persistent compatibility. 
When imposing $\mathcal{V}_c=\mathcal{V}$, we impose linear \classK functions in our simulation and experiment results. We also relax the constraint \eqref{eq::fs_cbf_condition_final_controller} with a slack variable

\subsection{Polytope Volume Estimation}
\label{section::volume_estimation}
In this section, we describe several ways to approximately compute the volume of feasible space.

\subsubsection{Monte Carlo}
\label{section::mc_volume}
 A common approach to estimating polytope volume is to evaluate the integral with Monte Carlo (MC) sampling\cite{lee2018convergence,emiris2018practical}. We randomly sample $K$ particles $p_i\in \mathcal{U}$ and evaluate the volume as follows
\eqn{
   \mathcal{V}(t,x) = \frac{\sum_{j=1}^K \mathds{1}(p_i\in \mathcal{U}_c) }{\textrm{vol}(\mathcal{U})}
   \label{eq::mc_volume}
}
where $\textrm{vol}(\mathcal{U})$ is assumed to be given and easy to compute practically as it is often defined as a symmetric polygon.

\subsubsection{Chebyshev Ball of a Polytope}
The center $c$ and radius $r$ of the largest sphere contained inside a polytope, called Chebyshev ball, can be computed with the following Linear Program (LP)\cite[Chapter 8]{boyd2004convex} 
\begin{subequations}
\eqn{
\min_{c,r} \quad & -r \\
\textrm{s.t.} \quad & a_i c + ||a_i|| r \leq b_i,~  \forall i\in \{1,2,..,N\}
}
\label{eq::circle_volume}
\end{subequations}
where $a_i$ and $b_i$ are the $i^{th}$ rows of polytope matrices $A,b$.

\subsubsection{Inscribing Ellipsoid}
An ellipoid $E$ is an image of the unit ball under affine transformation $(B,d)$, $
 E = \{ B u + d ~ | ~ ||u||\leq 1 \} $ where $B\in S_{++}^{m}$ and $d\in \reals^{m\times 1}$. 
 Using the fact that the ellipsoid volume is proportional to $\det B$, the maximum volume ellisoid can be obtained by solving the following convex optimization problem \cite[Chapter 8]{boyd2004convex}
\begin{subequations}
\eqn{
\min_{B, d} \quad & \log \det B^{-1} \label{eq::ellipse_objective}\\
\textrm{s.t.} \quad & || B a_i || + a_i d \leq b_i, ~ \forall i\in \{1,2,..,N\}
}
\label{eq::ellipse_volume}
\vspace{-4mm}
\end{subequations}

In practice, Monte Carlo provides most accurate estimates of polytope volume in our simulation studies. The simple MC method in \eqref{eq::mc_volume} however is known to perform poorly in high-dimensional spaces. Since Monte Carlo estimates are prone to be noisy, this method is also not very suitable for evaluating gradients, which are needed to enforce CBF conditions. While more accurate but relatively expensive MC variations for volume computations\cite{volestiVolume} are known to exist, we leave the investigation of their suitability for control applications to future work.  The implementation with circle approximation is fastest since it only requires solving an LP however it suffers from non-uniqueness of solutions and therefore the LP solver may converge to different solutions for the same state configuration (We comment more on the practical implications of this in Section \ref{section::simulation_results}). Moreover, the circle may severely underapproximate the polytope volume. The ellipsoid approximates the volume much more accurately than the circle but is slower to implement. In Section \ref{section::simulation_results}, we observe that even with 10 constraints, the SDP can be solved fast enough for real-time implementations.
\vspace{-1mm}
\subsection{Smooth Approximation of Polytope Volume}
The polytope volume $\mathcal{V}(t,x)$ is not a differentiable function in general. An example configuration that is non-differentiable is shown in Fig. \ref{fig::fs_example_classK}. Therefore, $\mathcal{V}(t,x)$ should ideally be formulated as a non-smooth CBF\cite{glotfelter2017nonsmooth} and generalized gradients of the exact volume given by \eqref{eq::polytope_volume} as well as of the ellipse and circle approximations should be used. We defer such an extensive non-smooth analysis to future work. In this work, we utilize gradients of circle and ellipse approximations that are computed using existing tools, details of which are mentioned in Section \ref{section::simulation_results}. In the remainder of this section, we propose a smooth underapproximation of \eqref{eq::polytope_volume}. 



\begin{theorem}
    Consider the feasible space $\mathcal{U}_c(t,x)$ defined in \eqref{eq::cvx_polytope_eqn} and its volume $\mathcal{V}(t,x)=\textrm{vol}(\mathcal{U}_c(t,x))$. Let $H:\reals\rightarrow \{0,1\}$ be the Heaveside unit step function such that $H(y)=0$ for $y<0$ and $H(y)=1$ otherwise. Let $H_s$ be a differentiable approximation of $H$ such that $H_s(z) \leq H(x),  \forall z \in \reals$. Define $\mathcal{V}_s:\reals^+\times \reals^n\rightarrow\reals$ as the following differentiable function
    \eqn{
    \mathcal{V}_s(t,x) = \int\limits_{u\in \reals^m} \left[  \prod_{i=1}^N H_s(b_i(t,x) - A_i(t,x) u)  \right] d\mathcal{V}
    }
    Then $\mathcal{V}(t,x)\geq \mathcal{V}_s(t,x)$ for all $t\in \reals^+,x\in \reals^n$.
    \vspace{-3mm}
\end{theorem}

\begin{proof}
    For brevity, we skip the explicit dependency on time in the following. The volume of the polytope \eqref{eq::cvx_polytope_eqn} is given by
\begin{subequations}
    \eqn{
\mathcal{V}(x) &= \textrm{vol}(\mathcal{U}_c(x)) 
= \int\limits_{u\in \{u \in \reals^m| A(x) u \leq b(x) \}} d\mathcal{V} \\
&= \int\limits_{u\in \reals^m} \left[  \prod_{i=1}^N H(b_i(x) - A_i(x) u)  \right] d\mathcal{V} \label{eq::volume_as_step_function}
}    
\end{subequations}
where $d\mathcal{V}$ is the infinitesimal volume (for example, $d\mathcal{V}=du_x du_y$ when $m=2$) 
Note that although $H$ is a discontinuous function, the integral \eqref{eq::volume_as_step_function} exists as the points of discontinuity, defined by the boundary of the polytope, form a set of measure 0. Furthermore, since $H(y)\leq H_s(y)$ for all $y\in \reals$, we have $\mathcal{V}(t,x)\geq \mathcal{V}_s(t,x)$.
\end{proof}
Note that, although smooth, evaluating the integral \eqref{eq::volume_as_step_function} analytically is still non-trivial and requires further numerical techniques such as MC approximation. Due to sensitivity of gradient to the number of samples chosen, we observe that MC, although suitable for obtaining volume, is not suitable for gradient computation. Therefore for lack of space, we only present results with ellipse and circle approximations in the ensuing. Methods to provide stable gradients through \eqref{eq::volume_as_step_function} will be explored in future work.

\section{Simulation Results}
\label{section::simulation_results}
In this section, we present two case studies to evaluate the performance of the proposed volume barrier function. All the simulations are run on a laptop with i9-13905 CPU. The dynamics are propagated using Euler integration with a time step of 0.01s. We use Python 3.11 and JAX's autodiff feature to compute all required gradients except for the circle and ellipse approximations whose details are mentioned in the next section. Before proceeding with the results, we first mention some important implementation details of our framework.

The sphere and ellipsoid underapproximations of $\mathcal{V}(t,x)$ are computed using LP and SDP respectively, which are convex programs and, therefore, we can exploit sensitivity analysis \cite{agrawal2019differentiable} to compute the gradients of their volume w.r.t the polytope matrices $A,b$. Libraries like cvxpy \cite{diamond2016cvxpy} readily provide an interface to differentiate through convex programs and we use cvxpy's JAX interface in our simulations. While LP and SDPs are not always differentiable, we do not encounter observable issues in our simulations. We hypothesize that this is because the configurations that are non-differentiable, which are related to instances when a new hyperplane enters or exits the feasible space, are rarely encountered exactly in a discrete-time simulation. A more thorough analysis though is subject to future work. Note that, if any non-differentiable configuration is encountered, cvxpy returns heuristic gradients. 

In all our simulations, we first obtain the polytope volume $\mathcal{V}$ and its gradients using one of the proposed methods: smooth approximation or circle/ellipsoid approximation. We then use these values to formulate the CBF conditions and solve the controller \eqref{eq::fs_final_cbf_controller}. Therefore, in the case of circles and ellipsoids, we solve two optimization problems sequentially: first, an LP/SDP, and then the FS-CBF-QP.

\subsection{Case Study 1: Valid FS-CBF computation}

\label{section::case_study1}
In this case study, we demonstrate the concept of FS-CBF for navigating in the obstacle environment shown in Fig. \ref{fig::si_fs}. A Dubins car model is considered: $\dot p_x = v\cos\theta, \dot p_y = v\sin\theta, \dot\theta = u$ where $p_x,p_y$ are position, $\theta$ is heading, $v$ is fixed speed, and the angular velocity $u\in [-0.5,0.5]$ is the control input.
We consider the constraint functions, also called candidate CBFs, $c_i=d_i^2-d_{min}^2$, with $d_i$ being the distance to the center of $i^{th}$ obstacle and $d_{min}$ the radius of obstacle. 
Next, we find valid CBFs $h_i$ to each $c_i$ using an existing method called \textit{refineCBF} proposed in \cite{tonkens2022refining} that leverages Hamilton-Jacobi (HJ) reachability analysis for designing CBFs.
Note that each valid CBF $h_i$ does not enforce collision avoidance with an obstacle $j\neq i$. Now we present results for two cases: In Case 1, we start with valid CBFs $h_i$ and show their compatible state space in Fig. \ref{fig::refinedCBFs}(a). The feasible space volume (of one-dimensional control space in this example) is computed using 100 MC samples using method in Section~\ref{section::mc_volume}. 
We once again use refineCBF to find a valid CBF on this compatible state space. 
The resulting FS-CBF is visualized in Fig. \ref{fig::refinedCBFs}(b). In Case 2, we start with the candidate CBFs $c_i$ and visualize their compatible state space in Fig. \ref{fig::refinedCBFs}(c),(e) for two different choices of linear \classK function parameters. 
Note that $c_i$ are higher-order barriers and their feasible space is defined by \eqref{eq::feasible_cbf1} for $r_i=2$. 
Next, we find a valid CBF on compatible state space of $c_i$ using the refineCBF framework and visualize the resulting FS-CBF in Fig. \ref{fig::refinedCBFs}(d),(f). Observe that compatibility is highly sensitive to chosen \classK parameters. This presents a significant challenge in practice wherein manually tuned candidate CBFs are often employed as finding a valid CBF is not always feasible computationally.

Note that the unsafe region (function value of zero in Fig.\ref{fig::refinedCBFs}) is larger in Fig.\ref{fig::refinedCBFs})(b),(d),(f) compared to Fig.\ref{fig::refinedCBFs})(a),(c),(e). This implies that some of the states with non-zero feasible space volume are outside the found forward invariant set and are therefore unsafe. We thus reiterate that CBFs designed independently do not possess any guarantee of compatibility when enforced together. 


\begin{figure}
    \centering
    \includegraphics[width=0.44\textwidth]{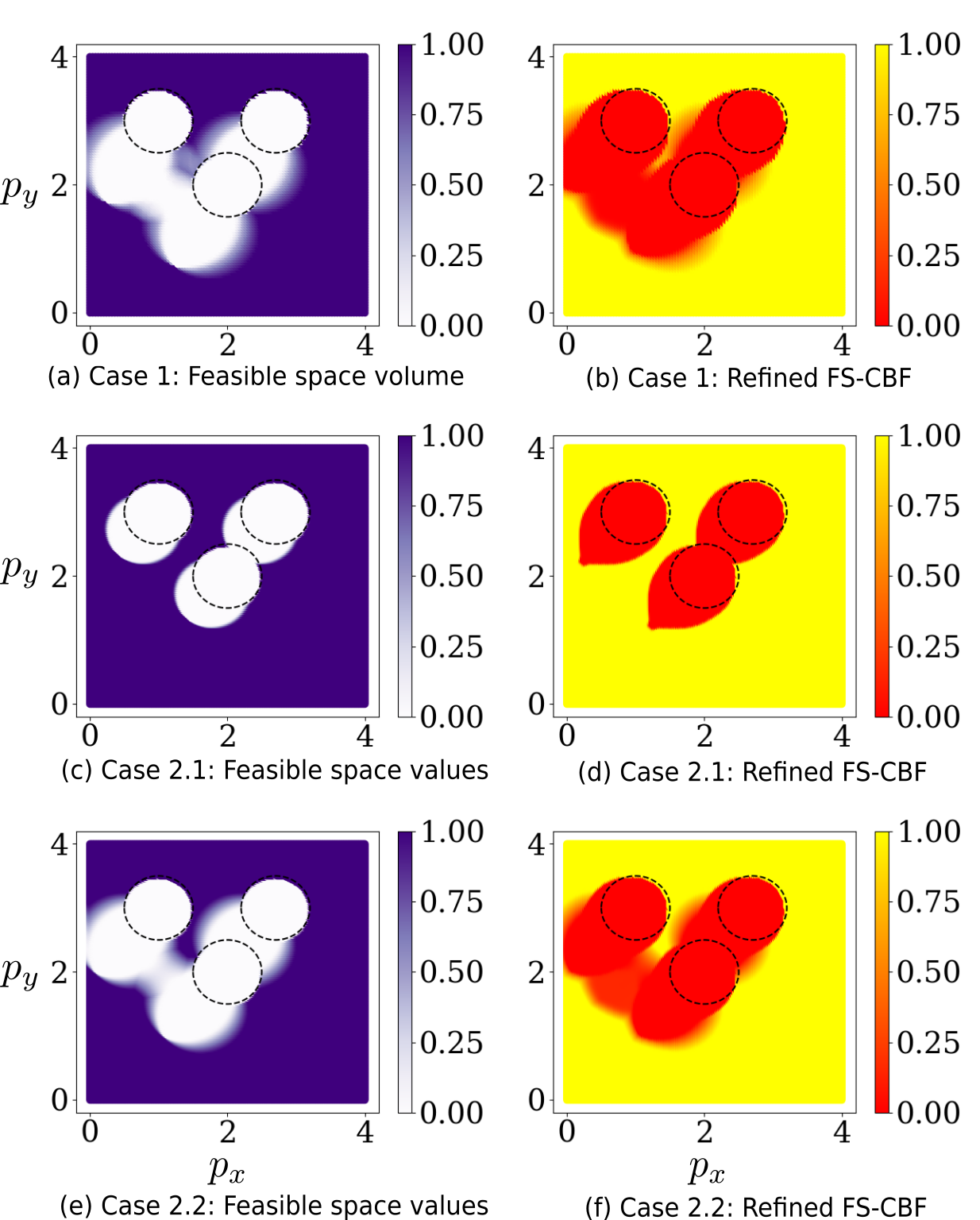}
    \caption{\small{Compatible space vs refined FS-CBF values at different $X,Y$ and heading $\theta = 45^0$. The color bar represents the numerical values of FS volume and FS-CBF clipped to the interval [0,1]. (a),(b) Case 1 - initiating with three \textit{valid} CBFs. (c)-(f) initiating with three \textit{candidate} CBFs with aggressive (Case 2) and conservative (Case 3) \classK function parameters for HOCBF. The compatible space (purple) is a subset of the forward invariant set (yellow) as also illustrated in Fig. \ref{fig::mainfig}  }}
    \label{fig::refinedCBFs}
    \vspace{-3mm}
\end{figure}

\subsection{Case Study 2: Improving Feasibility and Sensitivity}
\label{section::case_study_2}
\begin{figure}
    \centering
    \includegraphics[width=0.40\textwidth]{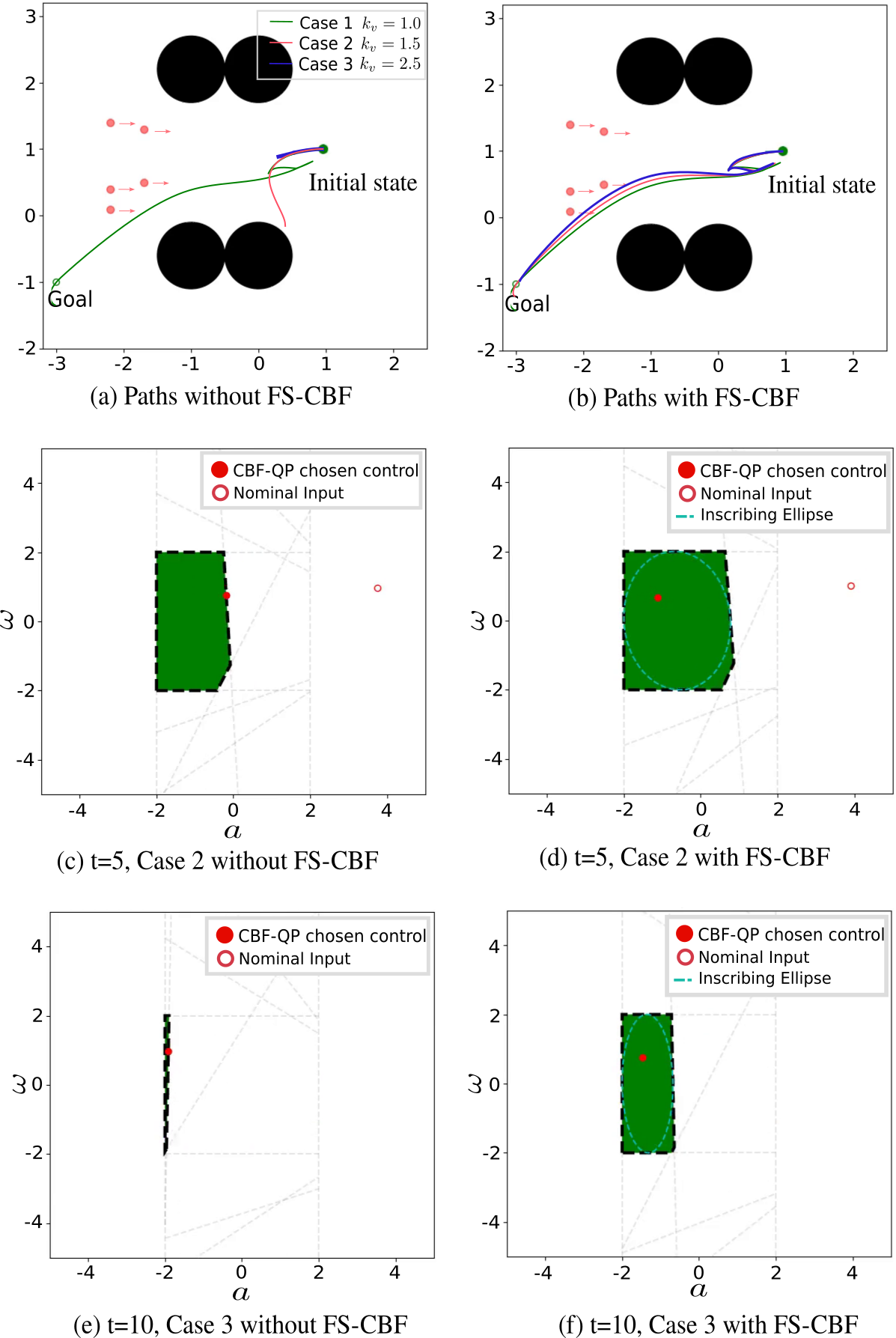}
    \caption{\small{
        Case study 1 - humans move towards the right. (a),(b) Paths for the three cases of $k_v=1,1.5,2.5$. Blue and red paths in (a) end when CBF-QP becomes infeasible. (c), (d), (e), (f) Feasible space visualization with and without volume barrier function for $k_v=1.5$ in (c), (d) and $k_v=2.5$ in (e), (f). Note when nominal input is outside the feasible space, the volume barrier function maps it to the interior of feasible space whereas CBF-QP \eqref{eq::qp_controller} maps it to a point on the boundary.
        }}
    \label{fig::cs1_scenario}
    \vspace{-4mm}
\end{figure}
Consider the scenario shown in Fig. \ref{fig::cs1_scenario} for a robot navigating to its goal location $g$ while avoiding black prohibited zones and performing collision avoidance with humans. 
This case study serves two purposes. First, we show that the addition of a feasible space volume barrier function enhances compatibility of candidate CBFs. 
Second, it also helps reduce sensitivity to the hyperparameters of the controller. 
The robot is modeled as a dynamic unicycle with position $p=[p_x,p_y]$, velocity $v$, heading angle $\psi$ and the following dynamics $\dot{p}_x = v \cos\psi,~ \dot{p}_y = v \sin\psi,~ \dot v = a, ~\dot \psi = \omega$. 
The control inputs are the linear acceleration $a$ and angular velocity $\omega$. 
The humans move with social force model \cite{helbing1995social}\cite{socialforce}, and consider the robot as one of the social agents. We consider the following reference (nominal) controller for the extended unicycle dynamics
\eqn{
    \psi_{ref} &= \arctan(\frac{g_y-p_y}{g_x-p_x}), ~ \psi_e = \psi - \psi_{ref}, ~ \omega_{ref} = k_\omega \psi_e \nonumber  \\
    v_{ref} & = k_x || p - g || \cos \psi_e, ~ a_{ref} = k_v (v - v_{ref}) 
}
where $k_\omega, k_p, k_v>0$ are user-chosen gains. The three candidate CBFs are designed in similar fashion as in Section \ref{section::case_study1}. Fig. \ref{fig::cs1_scenario} shows compares the performance of CBF-QP \eqref{eq::qp_controller} and FS-CBF-QP \eqref{eq::fs_final_cbf_controller} for three different values of the gain $k_v$. 
Fig. \ref{fig::cs1_scenario}(a)-(b) show the trajectories of the robot, and Figs. \ref{fig::cs1_scenario}(c)-\ref{fig::cs1_scenario}(e) show the variation of polytope volume under three different values of the gain $k_v$. We observe that when $k_v$ is increased, that is, the nominal controller is more aggressive, the CBF-QP becomes infeasible sooner. 
FS-CBF-QP on the other hand can maintain feasibility and the robot successfully reaches the goal. as seen in Figs. \ref{fig::cs1_scenario}(d), \ref{fig::cs1_scenario}(f). The visualization in Fig. \ref{fig::cs1_scenario} also shows that the nominal control input was not mapped to the boundary of the feasible set, but to its interior as mentioned in Remark \ref{remark::gauge_map}. The video is available on \href{https://www.youtube.com/playlist?list=PLJxod9x5m_V2R5c4MLCCwwe_xliJgzOr2}{\scriptsize{https://www.youtube.com/playlist?list=PLJxod9x5m\_V2R5c4MLCCwwe\_xliJgzOr2}}.

The importance of our controller thus lies in enhancing the performance of a given CBF-QP controller. To further corroborate our claim, we perform extensive Monte Carlo simulations by randomizing the robot's initial state $p_x,p_y,v$ and choice of gains $k_x, k_v$. We report the average simulation run times in Table \ref{table::case_study_stats}. The run time is defined as the minimum of 7 seconds (maximum duration of the simulation) and the time at which CBF-QP ($T_1$) or FS-CBF-QP ($T_2$, reported for different $\alpha_v$s) become infeasible. We present results for small as well as large perturbarion in the initial state and gains $k_x,k_v$. In both the cases, we observe that FS-CBF-QP significantly outperforms CBF-QP. Note that infeasibility is not unexpected even for FS-CBF-QP as we only employ candidate CBFs and assume, for sake of implementation, that $\mathcal{V}_c=\mathcal{V}$. Since $\mathcal{V}$ is not a valid CBF, our claim is strictly limited to empirical demonstration of enhanced feasibility of QP and reduced sensitivity to user parameters. Designing a CBF for such a complicated scenario with multiple obstacle and heterogenerous agents is an active area of research and no generic solutions exist yet to the best of our knowledge.

\begin{table}[h]
\setlength{\tabcolsep}{2.0pt}
\centering
\begin{tabular}{ c ||c  c c  c  c  c  c  c } 
 \hline
 & $T_1$ & $T_2$ & $\alpha_\mathcal{V}$ & $p_x$ & $p_y$ & $v$ & $k_p$ & $k_v$ \\ 
 \hline
 \hline 
\multirow{3}{*}{Scene 1} & \multirow{3}{*}{1.67} & 4.07  & 0.8 & \multirow{3}{*}{[0.9,1.1]} & \multirow{3}{*}{[0.9,1.1]} & \multirow{3}{*}{[1.2,1.4]} & \multirow{3}{*}{[0.9,1.1]} & \multirow{3}{*}{[1.5,1.6]} \\ 
      & & 4.37 &  1.0 \\
      & & 3.53 &  2.0 \\
      \hline
\multirow{3}{*}{Scene 2} & \multirow{3}{*}{1.29} & 3.04   & 0.8 & \multirow{3}{*}{[0.9,1.1]} & \multirow{3}{*}{[0.9,1.1]} & \multirow{3}{*}{[1.2,1.4]} & \multirow{3}{*}{[0.9,1.1]} & \multirow{3}{*}{[2.5,2.6]} \\ 
      & & 2.13 & 1.0  \\
      & & 1.86 & 2.0  \\
      \hline
 \multirow{3}{*}{Scene 3} & \multirow{3}{*}{1.64} & 4.74  & 0.8 & \multirow{3}{*}{[0.8,1.2]} & \multirow{3}{*}{[0.0,1.5]} & \multirow{3}{*}{[1.0,1.5]} & \multirow{3}{*}{[0.5,3.0]} & \multirow{3}{*}{[0.5,4.0]} \\ 
      & & 4.47  & 1.0  \\
      & & 3.65  & 2.0  \\
      \hline
 \end{tabular}
 \caption{\small{Sensitivity analysis of time to incompatibility for small (Scene 1,2) and large (Scene 3) perturbations of the initial state $p_x,p_y,v$ and gains $k_p, k_v$. Variables are sampled with uniform distribution within the specified [lower bound, upper bound].}}
 \label{table::case_study_stats}
 \vspace{-2mm}
 \end{table}

\subsection{Case Study 3: AWS Hospital}
\label{section::case_study_3}
Next, we evaluate our FS-CBF for navigating turtlebot3 robot in AWS Hospital Gazebo environment shown in Fig.~\ref{fig::experiment_scenario}. We first generate a map of the environment without humans using ROS2 navigation stack. This map is used by global planner to plan paths, as well as by the CBF-QP controller to perform collision avoidance with static obstacles. At each time, the robot finds the closest occupied grid cell in the directions $\theta = \beta 30^0, \beta\in\{1,2,..,12\}$, and formulates 12 CBF constraints correspondingly. The humans are moved in real-time using a modified version of social force model plugin \cite{gazebosfm}. Humans, however, are agnostic to robot's motion and do not contribute to collision avoidance.
The position and velocity of humans is extracted from Gazebo and made available to the robot, which then adds a CBF constraint for each human. 
The collision avoidance is enforced using distance as the higher-order barrier function of order two which is known to work well in the presence of a single constraint only. We use linear \classK functions $\alpha_1(h)=2h, \alpha_2(h)=6h$ for first and second-order HOCBF derivative conditions. 
The proposed FS CBF-QP is compared with CBF-QP in \eqref{eq::qp_controller}. 
We observe that CBF-QP leads to fewer collisions with humans by preventing the robot from getting too close to the humans in congested scenarios. A video of the experiment can be found at \href{https://www.youtube.com/playlist?list=PLJxod9x5m_V2R5c4MLCCwwe_xliJgzOr2}{\scriptsize{https://www.youtube.com/playlist?list=PLJxod9x5m\_V2R5c4MLCCwwe\_xliJgzOr2}}.


\section{Conclusion}
We propose a new barrier function that restricts the rate-of-change of the volume of feasible solution space of CBF-QP. We also provide empirical evaluation to corroborate our theoretical results, and to show effectiveness in improving feasibility and in reducing sensitivity to changes in other system modules such as nominal controller. Future work will involve performing experiments with robots and non-smooth control theoretic analysis of FS-CBF. Finally, monitoring feasible space volume may also help alleviate feasibility issues in other optimization based controllers such as MPC. This analysis will also be performed in future work.

\bibliographystyle{IEEEtran}
\bibliography{references.bib}

\end{document}